\documentclass{article}
\usepackage{ijcai21}

\usepackage{times}
\usepackage{soul}
\usepackage{url}
\usepackage[hidelinks]{hyperref}
\usepackage[utf8]{inputenc}
\usepackage[small]{caption}
\usepackage{graphicx}
\usepackage{amsmath}
\usepackage{amsthm}
\usepackage{booktabs}
\usepackage{color}
\usepackage{amssymb}
\usepackage{bm}
\usepackage{subfigure}
\usepackage{makecell}
\usepackage{threeparttable}
\usepackage{multirow}
\usepackage[linesnumbered,boxed,ruled,commentsnumbered]{algorithm2e}

\newtheorem{theorem}{Theorem}
\newtheorem{lemma}{Lemma}

\title{Pruning of Deep Spiking Neural Networks through Gradient Rewiring}

\author{
	Yanqi Chen$^{1,3}$
	\and
	Zhaofei Yu$^{1,2,3 *}$\and
	Wei Fang$^{1,3}$\and
	Tiejun Huang$^{1,2,3}$\And
	Yonghong Tian$^{1,3}$\thanks{Corresponding author}
	\affiliations
	$^1$Department of Computer Science and Technology, Peking University\\
	$^2$Institute for Artificial Intelligence, Peking University\\
	$^3$Peng Cheng Laboratory
	\emails
	\{chyq, yuzf12, fwei, tjhuang, yhtian\}@pku.edu.cn
}

\begin{document}
	
	\maketitle
	
	\begin{abstract}
		Spiking Neural Networks (SNNs) have been attached great importance due to their biological plausibility and high energy-efficiency on neuromorphic chips. As these chips are usually resource-constrained, the compression of SNNs is thus crucial along the road of practical use of SNNs. Most existing methods directly apply pruning approaches in artificial neural networks (ANNs) to SNNs, which ignore the difference between ANNs and SNNs, thus limiting the performance of the pruned SNNs. Besides, these methods are only suitable for shallow SNNs. In this paper, inspired by synaptogenesis and synapse elimination in the neural system, we propose gradient rewiring (Grad R), a joint learning algorithm of connectivity and weight for SNNs, that enables us to seamlessly optimize network structure without retraining. Our key innovation is to redefine the gradient to a new synaptic parameter, allowing better exploration of network structures by taking full advantage of the competition between pruning and regrowth of connections. The experimental results show that the proposed method achieves minimal loss of SNNs' performance on MNIST and CIFAR-10 datasets so far. Moreover, it reaches a $\sim$3.5\% accuracy loss under unprecedented 0.73\% connectivity, which reveals remarkable structure refining capability in SNNs. Our work suggests that there exists extremely high redundancy in deep SNNs. Our codes are available at \url{https://github.com/Yanqi-Chen/Gradient-Rewiring}. 
	\end{abstract}
	
	\section{Introduction}
	Spiking Neural Networks (SNNs), as the third generation of Artificial Neural Networks (ANN) \cite{maass1997networks}, have increasingly aroused researchers’ great interest in recent years due to their high biological plausibility, temporal information processing capability, and low power consumption on neuromorphic chips \cite{gerstner2002spiking,roy2019towards}.
	As these chips are usually resource-constrained, 
	the compression of SNNs is thus of great importance to lower computation cost to make the deployment on neuromorphic chips more feasible, and realize the practical application of SNNs.  
	
	Researchers have made a lot of efforts on the pruning of SNNs.
	Generally, most methods learn from the existing compression techniques in the field of ANNs and directly apply them to SNNs, such as directly truncating weak weights below the threshold \cite{10.3389/fnins.2016.00241,8325325,8894853} and using Lottery Ticket Hypothesis \cite{frankle2018the} in SNNs \cite{9053412}. Some works use other metrics, such as levels of weight development \cite{10.3389/fnins.2019.00405} and similarity between spike trains \cite{9023707}, as guidance for removing insignificant or redundant weights. Notwithstanding some spike-based metrics for pruning are introduced in these works, the results on shallow SNNs are not persuasive enough.
	
	There are some attempts to mimic the way the human brain learns to enhance the pruning of SNNs. Qi \textit{et al.} \shortcite{ijcai2018-221} enabled structure learning of SNNs by adding connection gates for each synapse. They adopted the biological plausible local learning rules, Spike-Timing-Dependent Plasticity (STDP), to learn gated connection while using Tempotron rules \cite{gutig2006the} for regular weight optimization. However, the method only applies to the SNN with a two-layer structure. Deng \textit{et al.} \shortcite{deng2019comprehensive} exploited the ADMM optimization tool to prune the connections and optimize network weights alternately. Kappel \textit{et al.} \shortcite{10.1371/journal.pcbi.1004485} proposed synaptic sampling to optimize network construction and weights by modeling the spine motility of spiking neurons as Bayesian learning. Bellec \textit{et al.} went a further step and proposed Deep Rewiring (Deep R) as a pruning algorithm for ANNs \shortcite{bellec2018deep} and Long Short-term Memory Spiking Neural Networks (LSNNs) \shortcite{NIPS2018_7359}, which was then deployed on SpiNNaker 2 prototype chips \cite{10.3389/fnins.2018.00840}.
	Deep R is an on-the-fly pruning algorithm, which indicates its ability to reach target connectivity without retraining.
	Despite these methods attempting to combine learning and pruning in a unified framework, they focus more on pruning, rather than the rewiring process. Actually, the pruned connection is hard to regrow, limiting the diversity of network structures and thus the performance of the pruned SNNs.
	
	In this paper, inspired by synaptogenesis and synapse elimination in the neural system \cite{huttenlocher1994synaptogenesis,petzoldt2014synaptogenesis}, we propose gradient rewiring (Grad R), a joint learning algorithm of connectivity and weight for SNNs, by introducing a new synaptic parameter to unify the representation of synaptic connectivity and synaptic weights.
	By redefining the gradient to the synaptic parameter, Grad R takes full advantage of the competition between pruning and regrowth of connections to explore the optimal network structure. To show Grad R can optimize synaptic parameters of SNNs, we theoretically prove that the angle between the accurate gradient and redefined gradient is less than $90^{\circ}$.
	We further derive the learning rule and realize the pruning of deep SNNs. We evaluate our algorithm on the MNIST and CIFAR-10 dataset and achieves state-of-the-art performance. On CIFAR-10, Grad R reaches a classification accuracy of 92.03\% (merely \textless1\% accuracy loss) when constrained to $\sim$5\% connectivity. Further, it achieves a classification accuracy of 89.32\% ($\sim$3.5\% accuracy loss) under unprecedented 0.73\% connectivity, which reveals remarkable structure refining capability in SNNs. Our work suggests that there exists extremely high redundancy in deep SNNs. 
	
	\section{Methods}
	In this section, we present the pruning algorithm for SNNs in detail. We first introduce the spiking neural model and its discrete computation mechanism. Then we propose the network structure of the SNNs used in this work. Finally, we present the gradient rewiring framework to learn the connectivity and weights of SNNs jointly and derive the detailed learning rule.
	
	\subsection{Spiking Neural Model}
	The spiking neuron is the fundamental computing unit of SNNs. 
	Here we consider the Leaky
	Integrate-and-Fire (LIF) model \cite{gerstner2002spiking}, which is one of the simplest spiking neuron models widely used in SNNs. The membrane potential $u(t)$ of the LIF neuron is defined as:
	\begin{equation}
		\label{eq:neuron}
		\begin{split}
			\tau_m \frac{\mathrm{d}u(t)}{\mathrm{d}t}=-(u(t)-u_{\text{rest}})+\sum\nolimits_i w_iI_i(t),\\
			\lim\limits_{\Delta t\to 0^+}u(t^f+\Delta t)=u_{\text{rest}},\text{ if } u(t^f)=u_{\text{th}},
		\end{split}
	\end{equation}
	where $I_i(t)$ denotes the input current from the $i$-th presynaptic neuron at time $t$, $w_i$ denotes the corresponding \textit{synaptic weight} (connection weight). $\tau_m$ is the membrane time constant, and $u_{\text{rest}}$ is the resting potential. When the membrane potential exceeds the firing threshold $u_{\text{th}}$ at time $t^f$, a spike is generated and transmitted to postsynaptic neurons. At the same time, the membrane potential $u(t) $ of the LIF neuron goes back to the reset value $u_{\text{rest}} < u_{\text{th}}$ (see Eq.~\eqref{eq:neuron}).
	
	For numerical simulations of LIF neurons in SNNs, we need to consider a version of the parameters dynamics that is discrete in time. Specifically, Eq.~\eqref{eq:neuron} is approximated by: 
	\begin{equation}\label{eq:neuron-discrete}
		\begin{split}
			m_{t+1}&= u_t+\frac{1}{\tau_m}\left( -(u_t-u_{\text{rest}})+\sum\nolimits_i w_iI_{i,t}\right),\\
			S_{t+1} &= H(m_{t+1}-u_{\text{th}}),\\
			u_{t+1}&= S_{t+1}u_{\text{rest}}+(1-S_{t+1})m_{t+1}.
		\end{split}
	\end{equation}
	Here $m_t$ and $u_t$ denote the membrane
	potential after neural dynamics and after the trigger
	of a spike at time-step $t$, respectively.
	$S_t$ is the binary spike at time-step $t$, which equals 1 if there is a spike
	and 0 otherwise. $H(\cdot)$ is the Heaviside step function, which is defined by
	$H(x)=
	\begin{cases}
		1 & x\geq 0\\
		0 & x<0
	\end{cases}$.
	Note that $H(0)=1$ is specially defined.
	
	\begin{figure}[t]
		\begin{center}
			\includegraphics[width=0.7 \linewidth]{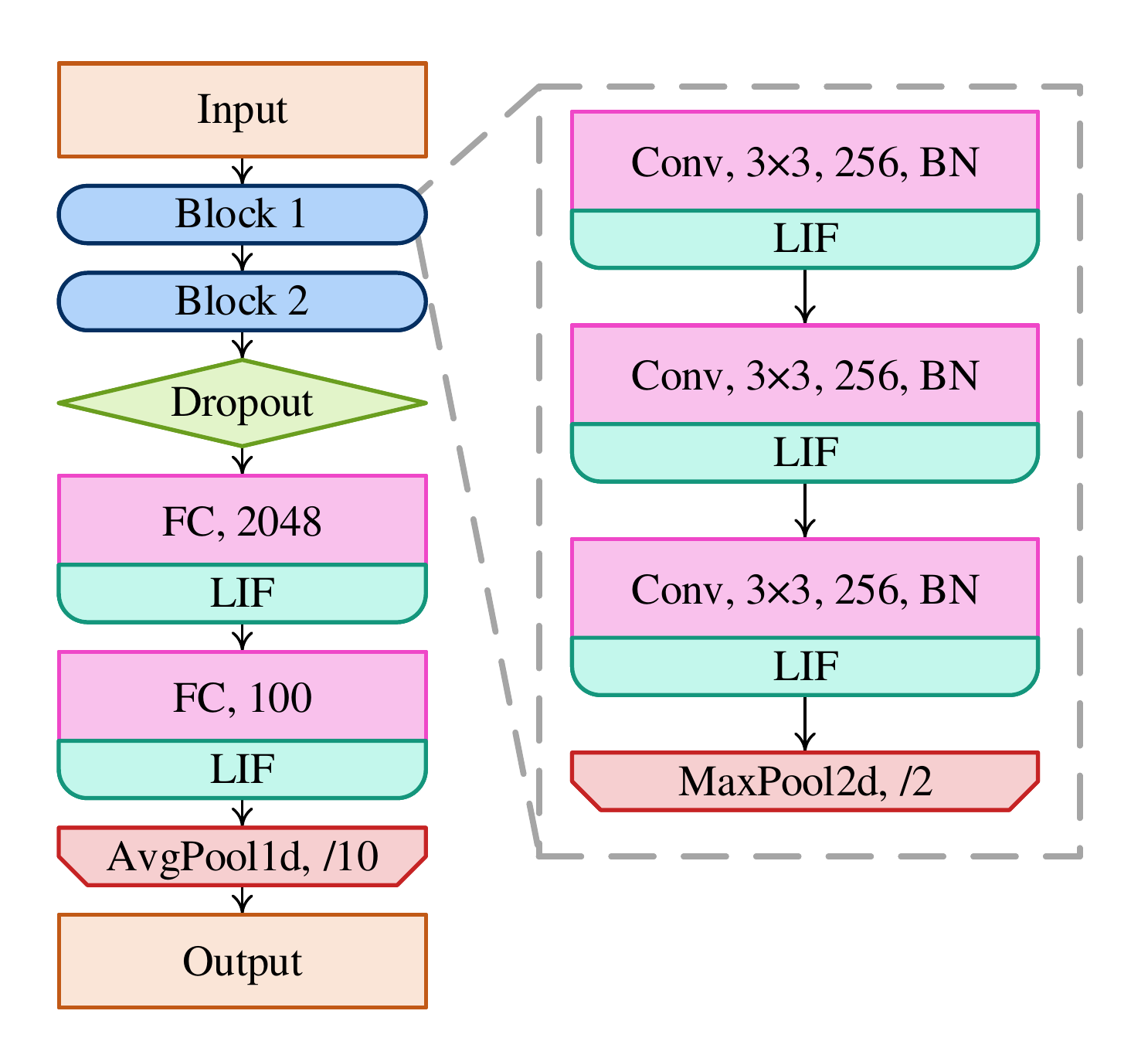}
		\end{center}
		\caption{The architecture of the SNN consisted of LIF neurons.}
		\label{fig:c10-arch}
	\end{figure}

	\subsection{Deep Spiking Neural Network}
	A typical architecture of deep SNNs is exemplified as Fig.~\ref{fig:c10-arch}, which consists of neuronal layers (LIF layers) generating spikes and synaptic layers including convolutional layers, pooling layers, and fully-connected (FC) layers. The synaptic layers define connective patterns of synapses between spiking neurons, which are leveraged for spatial correlative feature extraction or classification. A deep SNN such as CIFARNet proposed in \cite{Wu_Deng_Li_Zhu_Xie_Shi_2019} usually has multiple large-scale synaptic layers and requires pruning for reducing computation cost.

	\subsection{Gradient Rewiring Framework}
	Inspired by synaptogenesis and synapse elimination of the brain, we present a gradient-based rewiring framework for deep SNNs pruning and call it \textit{Gradient Rewiring}. Rewiring denotes the structural change of the brain’s wiring diagram. Gradient Rewiring combines pruning and learning to explore both connective patterns and weight spaces. The combination of pruning and learning requires continuous parameterized expression of wiring change (synaptic connection change) and weight optimization. 
	Previous neuroscience research has revealed that the wiring plasticity of neurons can be viewed as a special case of weight plasticity \cite{chklovskii2004cortical}.
	Thus we can unify the representation of synaptic connectivity and synaptic weights by introducing a new hidden synaptic parameter $\theta$ in addition to the synaptic weight $w$.
	
	The synaptic parameter $\theta$ is related to the volume of a dendritic spine, which defines the synaptic weights and connectivity for each synapse. A synapse is considered connected if $\theta>0$ and disconnected otherwise.
	Similar to \cite{bellec2018deep}, the synaptic parameter $\theta$ is mapped to synaptic weight through the relation $w=s \mathrm{ReLU}(\theta)$, where $s\in\{+1,-1\}$ is \textit{fixed} and used to distinguishes excitatory and inhibitory synapses. In this way, the sign change of synaptic parameter $\theta$ can describe both formation and elimination of synaptic connection.
	
	\subsubsection{Optimizing Synaptic Parameters \texorpdfstring{$\bm{\theta}$}{θ}}
	With the above definition, the gradient descent algorithm and its variants can be directly applied to optimize all synaptic parameters $\bm{\theta}$ of the SNNs. Specifically, we assume that the loss function has the form $\mathcal{L}(x_i,y_i,\bm{w})$, where $x_i,y_i$ denotes the $i$-th input sample and the corresponding target output, and $\bm{w}$ is the vector of all synaptic weights. The corresponding vectors of all synaptic parameters and signs are $\bm{\theta}, \bm{s}$ respectively. Suppose $w_k,\theta_k,s_k$ stand for the $k$-th element of $\bm{w}, \bm{\theta}, \bm{s}$. Then we have $\bm{w}=\bm{s}\odot \mathrm{ReLU}(\bm{\theta})$, where $\odot$ is the Hadamard product, and ReLU is an element-wise function. The gradient descent algorithm would adjust all the synaptic parameters to lower the loss function $\mathcal{L}$, that is: 
	\begin{align}\label{eq:grad-origin}
		\Delta \theta_k \propto \frac{\partial\mathcal{L}}{\partial \theta_k}=\frac{\partial \mathcal{L}}{\partial w_k}\frac{\partial w_k}{\partial \theta_k}=s_k\frac{\partial \mathcal{L}}{\partial w_k} H(\theta_k),
	\end{align}
	where $H(\cdot)$ represents the Heaviside step function. Ideally, if $\theta_k<0$ during learning, the corresponding synaptic connection will eliminate. It will regrow when $\theta_k>0$. 
	Nonetheless, the pruned connection will never regrow. 
	To see this, supposing that $\theta_k=\theta<0$, we get $H(
	\theta_k)=0$, and ${w_k=s_k\mathrm{ReLU}(\theta_k)=0}$. Then:
	\begin{align}\label{eq:grad-neg}
		\frac{\partial\mathcal{L}}{\partial \theta_k}\biggr|_{\theta_k=\theta}=H(\theta)s_k\frac{\partial \mathcal{L}}{\partial w_k}\biggr|_{w_k=0}=0 ,~~\text{for } \theta<0.
	\end{align}
	As the gradient of the loss function $\mathcal{L}$ with respect to $\theta_k$ is 0, $\theta_k$ will keep less than 0, and the corresponding synaptic connection will remain in the state of elimination.
	This might result in a disastrous loss of expression capabilities for SNNs. So far, the current theories and analyses can also systematically explain the results of deep rewiring \cite{bellec2018deep}, where the pruned connection is hard to regrow even noise is added to the gradient. There is no preference for minimizing loss when adding noise to gradient.

	\subsubsection{Gradient Rewiring}
	Our new algorithm adjusts $\bm{\theta}$ in the same way as gradient descent, but for each $\theta_k$, it uses a new formula to redefine the gradient when $\theta_k<0$.
	\begin{align}\label{eq:new-grad}
		(\Delta \theta_k)_{\text{GradR}} \propto	\frac{\partial\bar{\mathcal{L}}}{\partial \theta_k} := s_k\frac{\partial \mathcal{L}}{\partial w_k}\biggr|_{w_k=s_k \text{ReLU}(\theta_k)},~~\text{for } \theta_k \in\mathbb{R}.
	\end{align}
	For an active connection $\theta_k=\theta>0$, the definition makes no difference as
	$H(\theta_k)=1$. For an inactive connection $\theta_k=\theta<0$, the gradients are different as $\Delta \theta_k=0$ and $(\Delta \theta_k)_{\text{GradR}} \propto s_k\frac{\partial \mathcal{L}}{\partial w_k}\bigr|_{w_k=0}$.
	The change of gradient for inactive connections results in a different loss function $\bar{\mathcal{L}}$. As illustrated in Fig.~\ref{fig:grad:a}, when $s_k\frac{\partial \mathcal{L}}{\partial w_k}\bigr|_{w_k=0}<0$, the gradient of $\bar{\mathcal{L}}$ guides a way to jump out the plateau in $\mathcal{L}$ (red arrow). With the appropriate step size chosen, a depleted connection will regrow without increasing the real loss $\mathcal{L}$ on this occasion. It should be noted that the situation in Fig.~\ref{fig:grad:b} can also occur and push the parameter $\theta_k$ even more negative, which can be overcome by adding a constraint scope to the parameter $\bm{\theta}$ (discussed in the subsequent part).
	
	Similar to the gradient descent algorithm (Eq.~\eqref{eq:grad-origin}),
	Grad R can also optimize synaptic parameters $\bm{\theta}$ of the SNNs. An intuitive explanation is the angle between the accurate
	gradient $\nabla_{\bm{\theta}}\mathcal{L}$ and the redefined gradient $\nabla_{\bm{\theta}}\bar{\mathcal{L}}$ is below 90\textdegree, i.e., $\langle \nabla_{\bm{\theta}}\mathcal{L}, \nabla_{\bm{\theta}}\bar{\mathcal{L}} \rangle \geq 0$,
	which can be derived from Lemma \ref{lemma:1}.  
	Further, a rigorous proof is illustrated in Theorem \ref{theorem:1}.
	
	\begin{figure}[t]
		\begin{center}
			\subfigure[$s\frac{\partial \mathcal{L}}{\partial w}\Bigr|_{w=0}<0$]{
				\label{fig:grad:a} 
				\includegraphics[width=0.45\linewidth]{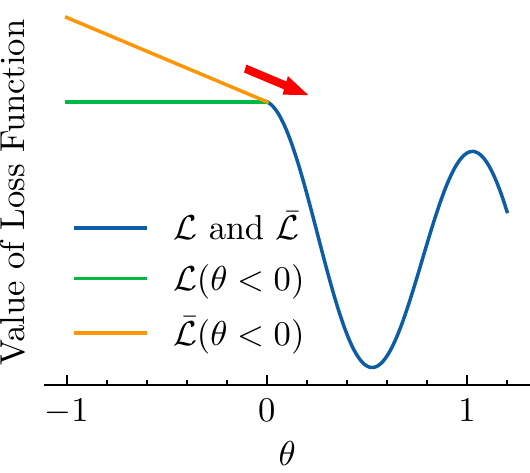}
			}
			\subfigure[$s\frac{\partial \mathcal{L}}{\partial w}\Bigr|_{w=0}>0$]{
				\label{fig:grad:b} 
				\includegraphics[width=0.45\linewidth]{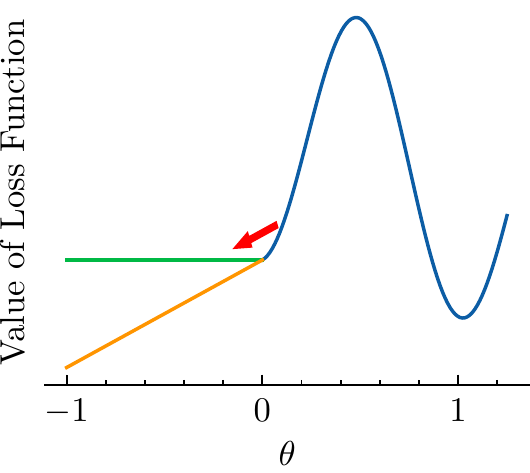}
			}
		\end{center}
		\caption{The comparison of gradient descent on $\bar{\mathcal{L}}$ and $\mathcal{L}$. The \textbf{red} arrows specify the direction of gradient descent on $\bar{\mathcal{L}}$ for $\theta$.}
		\label{fig:grad}
	\end{figure}
	
	\begin{theorem}[Convergence]
		\label{theorem:1}
		For a spiking neural network with synaptic parameters $\bm{\theta}$, synaptic weights $\bm{w}$, signs $\bm{s}$, and the SoftPlus mapping, a smooth approximation to the ReLU, $w_k=\frac{s_k}{\beta}\log(1+\exp(\beta \theta_k)),\beta\gg 1$, where $w_k,\theta_k,s_k$ stand for the $k$-th element of $\bm{w}, \bm{\theta}, \bm{s}$, the loss function $\mathcal{L}(\bm{\theta})$ with lowerbound, if there exists a constant $C$ such that $\forall k, \theta_k\geq C$ and a Lipschitz constant $L$ such that $\|\nabla_{\bm{\theta}}\mathcal{L}(\bm{\theta}_1)-\nabla_{\bm{\theta}}\mathcal{L}(\bm{\theta}_2)\|\leq L\|\bm{\theta}_1-\bm{\theta}_2\|$ for any differentiable points $\bm{\theta}_1$ and $\bm{\theta}_2$, then $\mathcal{L}$ will decrease until convergence using gradient $\nabla_{\bm{\theta}}\bar{\mathcal{L}}$, if step size $\eta>0$ is small enough.
	\end{theorem}
	
	\begin{proof}
		Let $\bm{\theta}^t$ be the parameter $\bm{\theta}$ after $t$-th update steps using $\nabla_{\bm{\theta}}\bar{\mathcal{L}}$ as direction for each step, that is, $\bm{\theta}^{t+1}=\bm{\theta}^t-\eta \nabla_{\bm{\theta}}\bar{\mathcal{L}}(\bm{\theta}^t)$. With Lipschitz continuous gradient, we have
		\begin{equation}\label{eq:theorem}
			\begin{split}
				&\mathcal{L}(\bm{\theta}^{t+1})-\mathcal{L}(\bm{\theta}^{t})\\
				=&\int_{0}^{1}\langle\nabla_{\bm{\theta}}\mathcal{L}(\bm{\theta}^{t}(\tau)),\bm{\theta}^{t+1}-\bm{\theta}^{t}\rangle\mathrm{d}\tau\\
				=&\langle\nabla_{\bm{\theta}}\mathcal{L}(\bm{\theta}^{t}),\bm{\theta}^{t+1}-\bm{\theta}^{t}\rangle\\
				&+\int_{0}^{1}\langle\nabla_{\bm{\theta}}\mathcal{L}(\bm{\theta}^{t}(\tau))-\nabla_{\bm{\theta}}\mathcal{L}(\bm{\theta}^{t}),\bm{\theta}^{t+1}-\bm{\theta}^{t}\rangle\mathrm{d}\tau\\
				\leq&\langle\nabla_{\bm{\theta}}\mathcal{L}(\bm{\theta}^{t}),\bm{\theta}^{t+1}-\bm{\theta}^{t}\rangle\\
				&+\int_{0}^{1}\|\nabla_{\bm{\theta}}\mathcal{L}(\bm{\theta}^{t}(\tau))-\nabla_{\bm{\theta}}\mathcal{L}(\bm{\theta}^{t})\|\cdot\|\bm{\theta}^{t+1}-\bm{\theta}^{t}\|\mathrm{d}\tau\\
				\leq&\langle\nabla_{\bm{\theta}}\mathcal{L}(\bm{\theta}^t),\bm{\theta}^{t+1}-\bm{\theta}^t\rangle+L\|\bm{\theta}^{t+1}-\bm{\theta}^t\|^2\int_{0}^{1}\tau\mathrm{d}\tau\\
				=&-\eta\langle\nabla_{\bm{\theta}}\mathcal{L}(\bm{\theta}^t),\nabla_{\bm{\theta}}\bar{\mathcal{L}}(\bm{\theta}^t)\rangle+\frac{L\eta^2}{2}\|\nabla_{\bm{\theta}}\bar{\mathcal{L}}(\bm{\theta}^t)\|^2\\
				\leq&-\eta(1-AL\eta/2)\langle\nabla_{\bm{\theta}}\mathcal{L}(\bm{\theta}^t),\nabla_{\bm{\theta}}\bar{\mathcal{L}}(\bm{\theta}^t)\rangle,
			\end{split}
		\end{equation}
		where $\bm{\theta}^{t}(\tau)$ is defined by $\bm{\theta}^{t}(\tau)=\bm{\theta}^{t}+\tau(\bm{\theta}^{t+1}-\bm{\theta}^{t}),\tau\in[0,1]$. The first two inequality holds by Cauchy-Schwarz inequality and Lipschitz continuous gradient, respectively. The last inequality holds with constant $A>0$ according to Lemma \ref{lemma:1}.
		
		When $\|\nabla_{\bm{\theta}}\bar{\mathcal{L}}(\bm{\theta}^t)\|>0$ and $\eta<2/AL$, the last term in Eq.~\eqref{eq:theorem} is strictly negative. Hence, we have a strictly decreasing sequence $\{\mathcal{L}(\bm{\theta}^n)\}_{n\in\mathbb{N}}$. Since the loss function has lowerbound, this sequence must converges.
	\end{proof}	
	
	\subsubsection{Prior-based Connectivity Control}
	There is no explicit control of connectivity in the above algorithm. We need to explore speed control of pruning targeting diverse levels of sparsity. Accordingly, we introduce a prior distribution to impose constraints on the range of parameters. To reduce the number of extra hyperparameters, we adopt the same Laplacian distribution for each synaptic parameter $\theta_k$: 
	\begin{align}\label{eq:prior}
		p(\theta_k)=\frac{\alpha}{2}\exp(-\alpha|\theta_k-\mu|),
	\end{align}
	where $\alpha$ and $\mu$ are scale parameter and location parameter respectively. The corresponding penalty term is:
	\begin{align}\label{eq:penalty}
		\frac{\partial \log p(\theta_k)}{\partial\theta_k}=-\alpha\mathrm{sign}(\theta_k-\mu),
	\end{align}
	which is similar to $\ell^1$-penalty with $\mathrm{sign}(\cdot)$ denoting the sign function.  
	
	Here we illustrate how the prior influences network connectivity. We assume that all the synaptic parameters $\bm{\theta}$ follow the same prior distribution in Eq.~\eqref{eq:prior}.  Since the connection with the negative synaptic parameter ($\theta_k <0 $) is pruned, the probability that a connection is pruned is: 
	\begin{equation}\label{eq:sparsity}
		\begin{split}
			\int_{-\infty}^{0}p(\theta_k)\mathrm{d}\theta_k&=\int_{-\infty}^{0}\frac{\alpha}{2}\exp\left( -\alpha|\theta_k-\mu|\right) \mathrm{d}\theta_k\\
			&=
			\begin{cases}
				\frac{1}{2}\exp(-\alpha\mu), & \text{if } \mu>0\\
				1-\frac{1}{2}\exp(\alpha\mu), & \text{if } \mu\leq 0
			\end{cases}.
		\end{split}
	\end{equation}
	As all the synaptic parameters follow the same prior distribution, Eq.~\eqref{eq:sparsity} also represents the network's sparsity. 
	Conversely, given target sparsity $p$ of the SNNs, one can get the appropriate choice of location parameter $\mu$ with Eq.~\eqref{eq:sparsity},
	which is related to the scale parameter $\alpha$
	\begin{equation}\label{eq:params}
		\mu=
		\begin{cases}
			-\frac{1}{\alpha}\ln(2p), & \text{if } p<\frac{1}{2}\\
			\frac{1}{\alpha}\ln(2-2p), & \text{if } p\geq\frac{1}{2}
		\end{cases}.
	\end{equation}
	Note that when $p<0.5$, we have $\mu>0$, which indicates that the peak of the density function $p(\theta_k)$ will locate at positive $\theta_k$. As the sign $s_k$ is chosen randomly and set fixed, the corresponding prior distribution $p(w_k)$ will be bimodal (explained in Appendix~\ref{sec:range}) . In view of that the distribution of weight (bias excluded) in convolutional and FC layers of a directly trained SNN is usually zero-meaned unimodal (See Fig. \ref{fig:param-no-prune}), we assume $p\geq 0.5$ in the following discussion.
	
	\begin{figure}[t]
		\begin{center}
			\includegraphics[width=0.95 \linewidth, trim=0 0 0 30,clip]{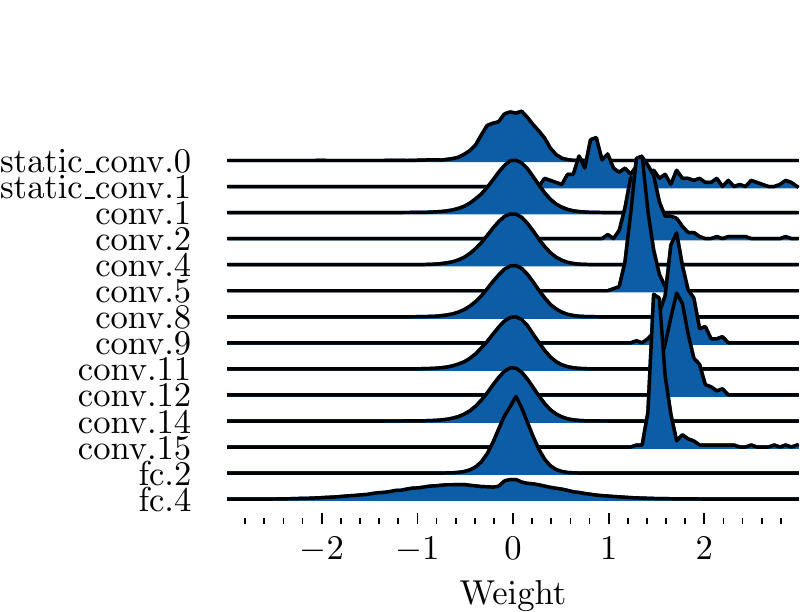}
		\end{center}
		\caption{The distribution of the parameters in a trained SNN without pruning. Note that \textit{static\_conv.1, conv.2, conv.5, conv.9, conv.12} and \textit{conv.15} are all BN layers.}
		\label{fig:param-no-prune}
	\end{figure}
	
	The Grad R algorithm is given in Algorithm \ref{algo:GR}.
	It is important to notice that the target sparsity $p$ is not bound to be reached. Since the gradient of the loss function, i.e., $s_k\frac{\partial \mathcal{L}(x_i,y_i,\bm{\theta})}{\partial w_k}$, also contributes to the posterior distribution (line \ref{alg:update}). Consequently, the choice of these hyperparameters is more a reference rather than a guarantee on the final sparsity. 
	The prior encourages synaptic parameters to gather around $\mu$, wherein the weights are more likely to be zero. Meanwhile, $\alpha$ controls the speed of pruning. The penalty term prevent weights from growing too large, alleviating the concern in Fig.~\ref{fig:grad:b}. 
	
	\begin{algorithm}[t]
		\caption{Gradient Rewiring with SGD}\label{algo:GR}
		\KwIn{Input data $X$, output labels $Y$, learning rate $\eta$, target sparsity $p$, penalty term $\alpha$}
		Initialize weight $\bm{w}$\;
		$\bm{s}\leftarrow\mathrm{sign}(\bm{w}), \bm{\theta}\leftarrow\bm{w}\odot\bm{s}$\;
		Calculate $\mu$ according to $p,\alpha$\;
		\For {t in $[1, N]$}{
			\For {all synaptic parameters $\theta_k$ and sign $s_k\in \bm{s}$}{
				\For {input $x_i\in X$ and label $y_i\in Y$}{
					$\theta_k\leftarrow \theta_k -\eta \left( s_k\frac{\partial \mathcal{L}(x_i,y_i,\bm{\theta})}{\partial w_k}+\alpha\mathrm{sign}(\theta_k-\mu)\right) $\;\label{alg:update}
				}
				$w_k\leftarrow s_k\mathrm{ReLU}(\theta_k)$\;
			}
		}
	\end{algorithm}

	\subsection{Detailed Learning Rule in SNNs}
	
	Here we apply Grad R to the SNN in Fig.~\ref{fig:c10-arch} and derive the detailed learning rule. Specifically, the problem is to compute the gradient of the loss function $s_k\frac{\partial \mathcal{L}(x_i,y_i,\bm{\theta})}{\partial w_k}$ (line \ref{alg:update} in Algorithm \ref{algo:GR}).
	We take a single FC layer without bias parameters as an example. Suppose that $\bm{m}_{t}$ and $\bm{u}_t$ represent the membrane potential of all the LIF neurons after neural dynamics and after the trigger of a spike at time-step $t$, 
	$\bm{S}_{t}$ denotes the vector of the binary spike of all the LIF neurons at time-step $t$,
	$\bm{W}$ denotes the weight matrix of FC layer, $\bm{1}$ denotes a vector composed entirely of 1 with the same size as $\bm{S}_{t}$.
	With these definitions, we can compute the gradient of output spikes $\bm{S}_{t}$ with respect to $\bm{W}$ by unfolding the network over the simulation time-steps, which is similar to the idea of backpropagation through time (BPTT). We obtain
	\begin{equation}\label{eq:fc-lif-grad}
		\begin{split}
			\frac{\partial \bm{S}_{t}}{\partial\bm{W}} &= \frac{\partial\bm{S}_{t}}{\partial\bm{m}_{t}}\frac{\partial \bm{m}_{t}}{\partial\bm{W}},\\
			\frac{\partial\bm{u}_{t}}{\partial\bm{W}} &= \frac{\partial\bm{u}_{t}}{\partial\bm{S}_t}\frac{\partial\bm{S}_{t}}{\partial\bm{W}}+\frac{\partial\bm{u}_{t}}{\partial\bm{m}_t}\frac{\partial\bm{m}_{t}}{\partial\bm{W}},\\
			\frac{\partial\bm{m}_{t+1}}{\partial\bm{W}} &= \frac{\partial\bm{m}_{t+1}}{\partial\bm{u}_{t}}\frac{\partial\bm{u}_{t}}{\partial\bm{W}}  + \frac{1}{\tau_m}\bm{I}_{t}^\top, 
		\end{split}
	\end{equation}
	where 
	\begin{equation}\label{eq:fc-lif-grad-2}
		\begin{split}
			\frac{\partial\bm{u}_{t}}{\partial\bm{S}_t} &= \mathrm{Diag}(u_{\text{rest}}\bm{1}-\bm{m}_{t}),\\
			\frac{\partial\bm{m}_{t+1}}{\partial\bm{u}_{t}} &= 1-\frac{1}{\tau_m},\\
			\frac{\partial\bm{u}_{t}}{\partial\bm{m}_t} &= \frac{\partial\bm{u}_{t}}{\partial\bm{S}_t}\frac{\partial\bm{S}_{t}}{\partial\bm{m}_t} + \mathrm{Diag}(\bm{1}-\bm{S}_{t}).
		\end{split}
	\end{equation}
	When calculating $\frac{\partial\bm{S}_{t}}{\partial\bm{m}_{t}}$, the derivative of spike with respect to the membrane potential after charging, we adopt the surrogate gradient method \cite{10.3389/fnins.2018.00331,8891809}. Specifically, we utilize
	shifted ArcTan function $H_1(x)=\frac{1}{\pi}\arctan(\pi x)+\frac12$ as a surrogate function of the Heaviside step function $H(\cdot)$, and we have
	$\frac{\partial\bm{S}_{t}}{\partial\bm{m}_{t}} \approx H_1'(\bm{m}_{t}-u_{\text{th}})$.
	
	\section{Experiments}
	We evaluate the performance of the Grad R algorithm on image recognition benchmarks, namely the MNIST and CIFAR-10 datasets. Our implementation of deep SNN are based on our open-source SNN framework \textit{SpikingJelly} \cite{SpikingJelly}.

	\subsection{Experimental Settings}
	For MNIST, we consider a shallow fully-connected network (Input-Flatten-800-LIF-10-LIF) similar to previous works. On the CIFAR-10 dataset, we use a convolutional SNN with 6 convolutional layers and 2 FC layers (Input-[[256C3S1-BN-LIF]$\times$3-MaxPool2d2S2]$\times$2-Flatten-Dropout-2048-LIF-100-LIF-AvgPool1d10S10, illustrated in Fig.~\ref{fig:c10-arch}) proposed in our previous work \cite{fang2020incorporating}. The first convolutional layer acts as a learnable spike encoder to transform the input static images to spikes. Through this way, we elude from the Poisson encoder that requires a long simulation time-step to restore inputs accurately \cite{Wu_Deng_Li_Zhu_Xie_Shi_2019}. The dropout mask is initiated at the first time-step and remains constant within an iteration of mini-batch data, guaranteeing a fixed connective pattern in each iteration \cite{10.3389/fnins.2020.00119}. All bias parameters are omitted except for BatchNorm (BN) layers. Note that pruning weights in BN layers is equivalent to pruning all the corresponding input channels, which will lead to an unnecessary decline in the expressiveness of SNNs. Moreover, the peak of the probability density function of weights in BN layers are far from zero (see Fig.~\ref{fig:param-no-prune}), which suggest the pruning is improper. Based on the above reasons, we exclude all BN layers from pruning in our algorithm.
	
	To facilitate the learning process, we use the Adam optimizer along with our algorithm for joint optimization of connectivity and weight. Furthermore, the gradients of reset $\frac{\partial u_t}{\partial S_t}$ in LIF neurons' forward steps (Eq.~\ref{eq:neuron-discrete}) are dropped, i.e. manually set to zero, according to \cite{Zenke2020.06.29.176925}. The choice of all hyperparameters is shown in Table \ref{tab:hyper-1} and Table \ref{tab:hyper-2}.
	
	\begin{figure}[t]
		\begin{center}
			\includegraphics[width=0.95\linewidth]{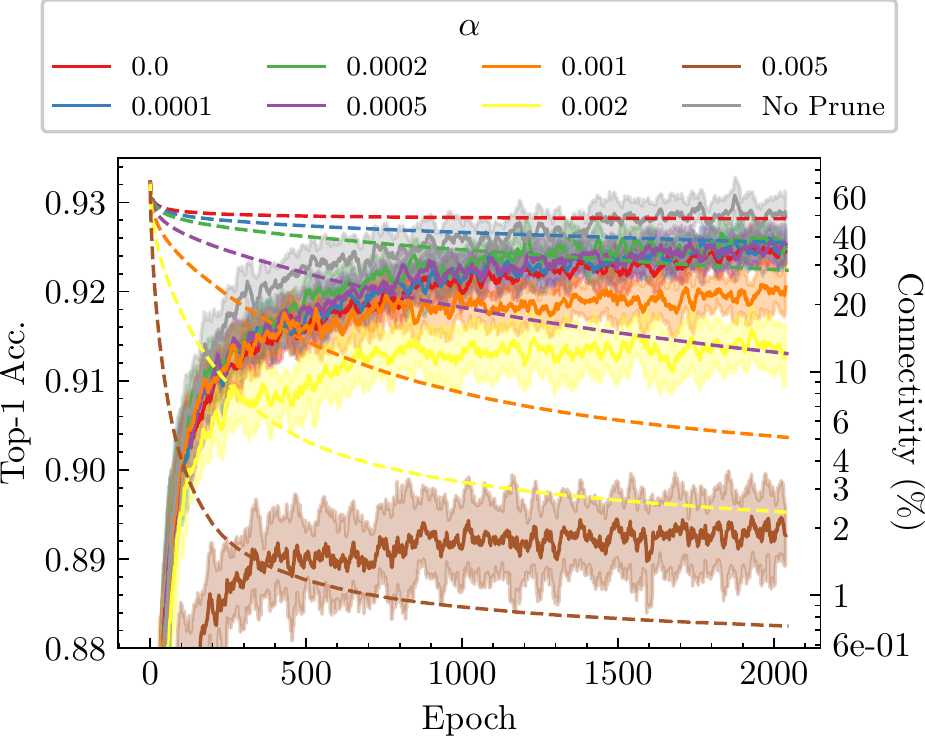}
		\end{center}
		\caption{Top-1 accuracy (solid line) on the test set and connectivity (dashed line) during training on CIFAR-10. $\alpha$ is the penalty term in Algorithm 1.}
		\label{fig:acc-connectivity-t}
	\end{figure}

	\subsection{Performance}
	
	\begin{table*}[h]
		\centering
		\scalebox{0.95}
		{
			\begin{threeparttable}
				\setlength{\tabcolsep}{2mm}{
					\begin{tabular}{cccrrr}
						\toprule
						
						\textbf{Pruning Methods} & \textbf{Training Methods} & \textbf{Arch.} & \textbf{Acc. (\%)} & \textbf{Acc. Loss (\%)} & \textbf{Connectivity (\%)} \\
						\midrule
						Threshold-based \cite{10.3389/fnins.2016.00241} & Event-driven CD & 2 FC & 95.60 & -0.60 & 26.00 \\
						Deep R \cite{NIPS2018_7359}\tnote{1} & Surrogate Gradient & LSNN\tnote{2} & 93.70 & +2.70 & 12.00 \\
						Threshold-based \cite{10.3389/fnins.2019.00405} & STDP & 1 FC & 94.05 & $\approx$-19.05\tnote{3} & $\approx$30.00 \\
						Threshold-based \cite{8325325} & STDP & 2 layers\tnote{4} & 93.20 & -1.70 & 8.00 \\
						\hline
						\multirow{3}*{ADMM-based \cite{deng2019comprehensive}} & \multirow{3}*{Surrogate Gradient} & \multirow{3}*{LeNet-5 like} & \multirow{3}*{99.07} & +0.03 & 50.00 \\
						&  &  &  & -0.43 & 40.00 \\
						&  &  &  & -2.23 & 25.00 \\
						\hline
						Online APTN \cite{10.3389/fnins.2020.598876} & STDP & 2 FC & 90.40 & -3.87 & 10.00 \\
						\hline
						\multirow{4}*{Deep R\tnote{5}} & \multirow{4}*{Surrogate Gradient} & \multirow{4}*{2 FC} & \multirow{4}*{98.92} & -0.36 & 37.14 \\
						&  &  &  & -0.56 & 13.30 \\
						&  &  &  & -2.47 & 4.18 \\
						&  &  &  & -9.67 & 1.10 \\
						\hline
						\multirow{5}*{\textbf{Grad R}} & \multirow{5}*{Surrogate Gradient} & \multirow{5}*{2 FC}  & \multirow{5}*{98.92} & \textbf{-0.33} & \textbf{25.71} \\
						&  &  &  & \textbf{-0.43} & \textbf{17.94} \\
						&  &  &  & \textbf{-2.02} & \textbf{5.63} \\
						&  &  &  & \textbf{-3.55} & \textbf{3.06} \\
						&  &  &  & \textbf{-8.08} & \textbf{1.38} \\
						\bottomrule	
				\end{tabular}}
				\begin{tablenotes}
					\footnotesize
					\item[1] Using sequential MNIST.
					\item[2] 100 adaptive and 120 regular neurons.
					\item[3] A data point of conventional pruning. Since the soft-pruning will only freeze weight during training, bringing no real sparsity.
					\item[4] The 1st layer is FC and the 2nd layer is a one-to-one layer with lateral inhibition.
					\item[5] Our implementation of Deep R.
				\end{tablenotes}
			\end{threeparttable}
		}
		
		\caption{Performance comparison between the proposed method and previous work on MNIST dataset.}
		
		\label{tab:acc-compare-MNIST}
	\end{table*}
	
	\begin{table*}[h]
		\centering
		\scalebox{0.95}
		{
			\begin{threeparttable}
				\setlength{\tabcolsep}{2.5mm}{
					\begin{tabular}{cccrrr}
						\toprule
						
						\textbf{Pruning Methods} & \textbf{Training Methods} & \textbf{Arch.} & \textbf{Acc. (\%)} & \textbf{Acc. Loss (\%)} & \textbf{Connectivity (\%)} \\
						\midrule
						\multirow{3}*{ADMM-based \cite{deng2019comprehensive}} & \multirow{3}*{Surrogate Gradient} & \multirow{3}*{7 Conv, 2 FC} & \multirow{3}*{89.53} & -0.38 & 50.00 \\
						&  &  &  & -2.16 & 25.00 \\
						&  &  &  & -3.85 & 10.00 \\
						\hline
						\multirow{3}*{Deep R\tnote{1}}& \multirow{3}*{Surrogate Gradient} & \multirow{3}*{6 Conv, 2 FC} & \multirow{3}*{92.84} & -1.98 & 5.24 \\
						&  &  &  & -2.56 & 1.95 \\
						&  &  &  & -3.53 & 1.04 \\
						\hline
						\multirow{5}*{\textbf{Grad R}} & \multirow{5}*{Surrogate Gradient} & \multirow{5}*{6 Conv, 2 FC}  & \multirow{5}*{92.84} & \textbf{-0.30} & \textbf{28.41} \\
						&  &  &  & \textbf{-0.34} & \textbf{12.04} \\
						&  &  &  & \textbf{-0.81} & \textbf{5.08} \\
						&  &  &  & \textbf{-1.47} & \textbf{2.35} \\
						&  &  &  & \textbf{-3.52} & \textbf{0.73} \\
						\bottomrule	
				\end{tabular}}
				\begin{tablenotes}
					\footnotesize
					\item[1] Experiment on CIFAR-10 is not involved in SNNs' pruning work of Deep R~\cite{NIPS2018_7359}. Therefore, we implement and generalize Deep R to deep SNNs.
				\end{tablenotes}
			\end{threeparttable}
		}
		\caption{Performance comparison between the proposed method and previous work on CIFAR-10 dataset.}
		
		\label{tab:acc-compare-c10}
	\end{table*}
	
	Fig.~\ref{fig:acc-connectivity-t} shows the top-1 accuracy on the test set and connectivity during training on the CIFAR-10 datasets.
	The improvement of accuracy and network sparsity is simultaneous as the Grad R algorithm can learn the connectivity and weights of the SNN jointly.
	As the training goes into later stages, the growth of sparsity has little impact on the performance. The sparsity will likely continue to increase, and the accuracy will still maintain, even if we increase the training epochs. The effect of prior is also clearly shown in Fig.~\ref{fig:acc-connectivity-t}. Overall, a larger penalty term $\alpha$ will bring higher sparsity and more apparent performance degradation. Aggressive pruning setting like huge $\alpha$, e.g. $\alpha=0.005$, will soon prune most of the connections and let the network leaves few parameters to be updated (see the brown line).
	
	As most of the previous works do not conduct experiments on datasets containing complicated patterns, such as CIFAR-10, we also test the algorithm on a similar shallow fully-connected network for the classification task of the MNIST dataset. We further implement Deep R as a baseline, which also can be viewed as a pruning-while-learning algorithm. The highest performance on the test set against connectivity is presented in Fig.~\ref{fig:res}. The performance of Grad R is comparable to Deep R on MNIST and much better than Deep R on CIFAR-10. As illustrated in Fig.~\ref{fig:res-c10}, the performance of Grad R is close to the performance of the fully connected network when the connectivity is above 10\%. Besides, it reaches a classification accuracy of 92.03\% (merely \textless1\% accuracy loss) when constrained to $\sim$5\% connectivity. Further, Grad R achieves a classification accuracy of 89.32\% ($\sim$3.5\% accuracy loss) under unprecedented 0.73\% connectivity. These results suggest that there exists extremely high redundancy in deep SNNs. We compare our algorithm to other state-of-the-art pruning methods of SNNs, and the results are listed in Table \ref{tab:acc-compare-MNIST} and Table \ref{tab:acc-compare-c10}, from which we can find that the proposed method outperforms previous works.

	\begin{figure}[t]
		\begin{center}
			\subfigure[MNIST]{
				\label{fig:res-mnist} 				
				\includegraphics[width=0.47\linewidth]{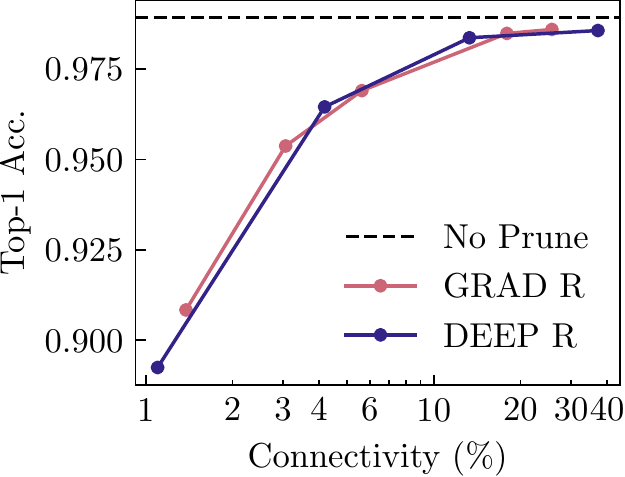}
			}
			\subfigure[CIFAR-10]{
				\label{fig:res-c10} 			    
				\includegraphics[width=0.47\linewidth]{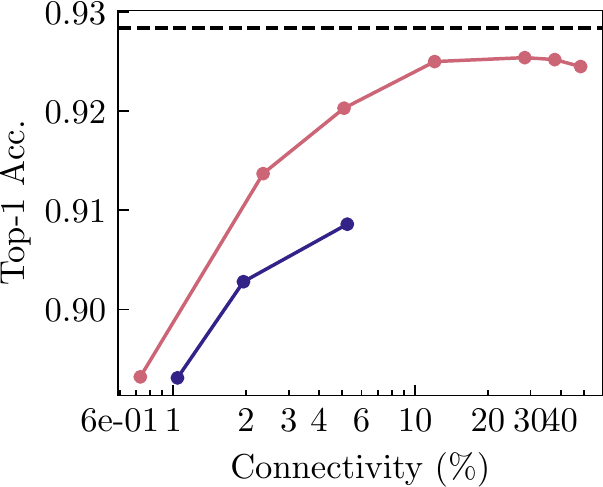}
			}
		\end{center}
		\caption{Top-1 accuracies on test set. Different penalty term $\alpha$ is chosen for different sparsity (connectivity). The accuracy of unpruned model is shown as a black dashed line.}
		\label{fig:res}
	\end{figure}

	\begin{figure}[t]
		\begin{center}
			\subfigure[Distribution of $\theta$]{
				\label{fig:ridge-theta} 
				\includegraphics[width=0.47\linewidth]{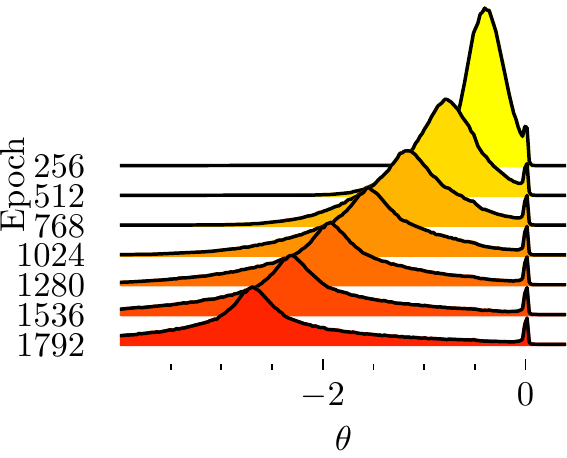}
			}
			\subfigure[Number of pruned and regrown connections]{
				\label{fig:prune-grow} 
				\includegraphics[width=0.47\linewidth]{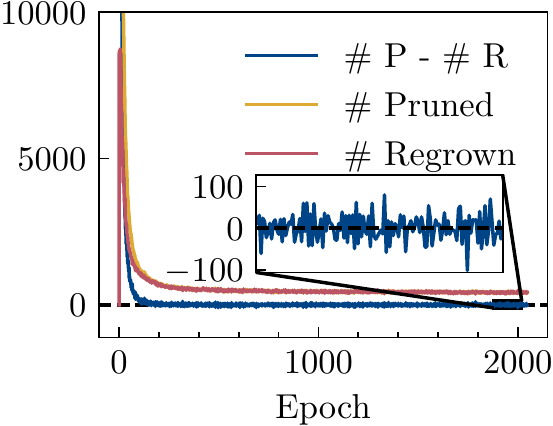}
			}
		\end{center}
		\caption{The statistics of the 2nd convolutional layer during training.}
		\label{fig:2nd-conv}
	\end{figure}	
	
	\begin{figure}[tb]
		\begin{center}
			\includegraphics[width=0.95\linewidth, trim=0 0 0 20,clip]{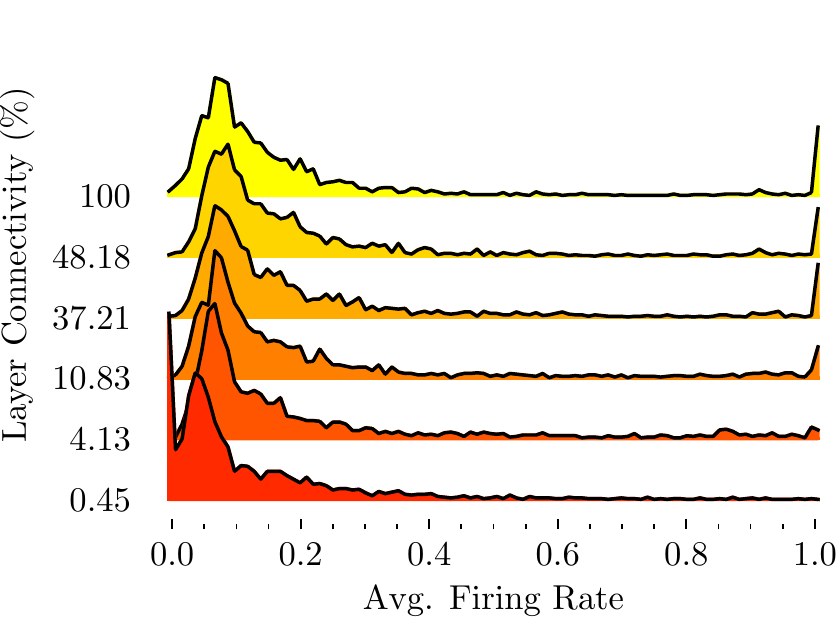}
		\end{center}
		\caption{The distribution of the average firing rate of post-synaptic neurons in the 1st FC layer under different connectivity. The firing rate of each neuron is averaged on samples of the whole test dataset and on time domain.}
		\label{fig:firing-rate-dist}
	\end{figure}

	\subsection{Behaviors of Algorithm}
	\noindent
	\paragraph{Ablation study of prior.}
	We test the effects of prior. The Grad R algorithm without prior ($\alpha=0$) achieves almost the highest accuracy after pruning, but the connectivity drops the most slowly (the red curve in Fig.~\ref{fig:acc-connectivity-t}).
	The performance gap between unpruned SNN and pruned SNN without prior suggests the minimum accuracy loss using Grad R. It also clarifies that prior is needed if higher sparsity is required.
	
	\noindent
	\paragraph{Distribution drift of parameters.}Fig.~\ref{fig:ridge-theta} illustrates the distribution of the parameter $\theta$ during training ($\alpha=0.005$). Since we use a Laplacian distribution with $\mu<0$ as prior, the posterior distribution will be dragged to fit the prior.
	
	\noindent
	\paragraph{Competition between pruning and regrowth.}
	A cursory glance of the increasing sparsity over time may mislead one that the pruning process always gains huge ascendancy over the regrowth process. However, Fig.~\ref{fig:prune-grow} sheds light on the fierce competition between these two processes. The pruning takes full advantage of the regrowth merely in the early stages (the first 200 epochs). Afterward, the pruning process only wins by a narrow margin. The competition encourages SNNs to explore different network structures during training.
	
	\noindent
	\paragraph{Firing rate and connectivity.}
	Reduction of firing rate is essential for energy-efficient neuromorphic chips (event-driven hardware) \cite{4252769}. Therefore, we explore the impact of synaptic pruning on the firing rate. As shown in Fig.~\ref{fig:firing-rate-dist},
	there is a relatively weak correlation between the distribution of neurons' average firing rate and connectivity. 
	However, when the connectivity drastically drops extremely low under 1\%, the level of neural activities is severely suppressed. The peak on the left side has a significant shift, causing by silent neurons. Meanwhile, the right peak representing saturated neurons vanishes.

	\section{Conclusion}
	In this paper, we present Gradient Rewiring, a pruning method for deep SNNs, enabling efficient joint learning of connectivity and network weight. It utilizes the competition between pruning and regrowth to strengthen the exploration of network structures and achieves the minimal loss of SNN's performance on the MNIST and CIFAR-10 dataset so far.  Our work reveals that there exists high redundancy in deep SNNs.
	An interesting future direction is to further investigate and exploit such redundancy to achieve low-power computation on neuromorphic hardware.	
	
	\appendix
	\section{Supporting Lemma}\label{sec:lemma}
	\begin{lemma}
		\label{lemma:1}
		If there exists a constant $C$, such that $\forall k, \theta_k\geq C$, then there exists a constant $A>0$ such that
		\begin{equation}
			\|\nabla_{\bm{\theta}}\bar{\mathcal{L}}\|^2\leq A\langle\nabla_{\bm{\theta}}\mathcal{L},\nabla_{\bm{\theta}}\bar{\mathcal{L}}\rangle.
		\end{equation}
	\end{lemma}
	\begin{proof}
		Since $w_k=\frac{s_k}{\beta}\log(1+\exp(\beta \theta_k))$, we have $\frac{\partial w_k}{\partial \theta_k}=s_k\sigma(\beta\theta_k)$, where $\sigma(x)=\frac{1}{1+\exp(-x)}$ is sigmoid function.
		
		Note that $\sigma(x)>0$ is monotonically increasing, we have
		\begin{equation}
			\begin{split}
				\frac{\|\nabla_{\bm{\theta}}\bar{\mathcal{L}}\|^2}{\langle\nabla_{\bm{\theta}}\mathcal{L},\nabla_{\bm{\theta}}\bar{\mathcal{L}}\rangle}&=\frac{\sum_k\left( s_k\frac{\partial\mathcal{L}}{\partial w_k}\right)^2}{\sum_k\left[ s_k^2\sigma(\beta\theta_k)\left( \frac{\partial\mathcal{L}}{\partial w_k}\right)^2\right]}\\
				&=\frac{\sum_k\left( \frac{\partial\mathcal{L}}{\partial w_k}\right)^2}{\sum_k\left[ \sigma(\beta\theta_k)\left( \frac{\partial\mathcal{L}}{\partial w_k}\right)^2\right]}\\
				&\leq\frac{\sum_k\left( \frac{\partial\mathcal{L}}{\partial w_k}\right)^2}{\sum_k\left[ \sigma(\beta C)\left( \frac{\partial\mathcal{L}}{\partial w_k}\right)^2\right]}\\
				&=\frac{1}{\sigma(\beta C)}.
			\end{split}
		\end{equation}
		Thus there exists a constant $A=\frac{1}{\sigma(\beta C)}$ such that
		$\|\nabla_{\bm{\theta}}\bar{\mathcal{L}}\|^2\leq A\langle\nabla_{\bm{\theta}}\mathcal{L},\nabla_{\bm{\theta}}\bar{\mathcal{L}}\rangle$.\\
	\end{proof}	
	Lemma \ref{lemma:1} also suggests that $\langle \nabla_{\bm{\theta}}\mathcal{L}, \nabla_{\bm{\theta}}\bar{\mathcal{L}} \rangle \geq 0$.
	When $\|\nabla_{\bm{\theta}}\bar{\mathcal{L}}(\bm{\theta}^t)\|>0$, the inner product $\langle \nabla_{\bm{\theta}}\mathcal{L}, \nabla_{\bm{\theta}}\bar{\mathcal{L}} \rangle$ is strictly positive, thus the angle between the accurate
	gradient $\nabla_{\bm{\theta}}\mathcal{L}$ and the redefined gradient $\nabla_{\bm{\theta}}\bar{\mathcal{L}}$ is below 90\textdegree.	
	
	\begin{table}[ht]
		\centering
		\scalebox{0.95}
		{
			\begin{tabular}{cccc}
				\toprule
				Parameters & Descriptions & CIFAR-10 & MNIST \\
				\midrule
				$N$ & \# Epoch & 2048 & 512\\
				Batch Size & - & 16 & 128\\
				$T$ & \# Timestep & \multicolumn{2}{c}{8} \\
				$\tau_m$ & Membrane Constant & \multicolumn{2}{c}{2.0} \\
				$u_{\text{th}}$ & Firing Threshold & \multicolumn{2}{c}{1.0} \\
				$u_{\text{rest}}$ & Resting Potential & \multicolumn{2}{c}{0.0} \\
				$p$ & Target Sparsity & \multicolumn{2}{c}{95\%} \\
				$\eta$ & Learning Rate & \multicolumn{2}{c}{1e-4}\\
				$\beta_1,\beta_2$ & Adam Decay & \multicolumn{2}{c}{0.9, 0.999} \\
				\bottomrule
			\end{tabular}
		}
		\caption{Choice of hyperparameters for both pruning algorithm.}
		\label{tab:hyper-1}
	\end{table}
	
	\begin{table}[h]
		\centering
		\begin{tabular}{cccc}
			\toprule
			Dataset & Algorithm & Penalty & Connectivity (\%) \\
			\midrule
			\multirow{10}*{CIFAR-10} & \multirow{7}*{Grad R} & $0$ & 48.35 \\
			& & $10^{-4}$ & 37.79 \\
			& & $2\times 10^{-4}$ & 28.41 \\
			& & $5\times 10^{-4}$ & 12.04 \\
			& & $10^{-3}$ & 5.08 \\
			& & $2\times 10^{-3}$ & 2.35 \\
			& & $5\times 10^{-3}$ & 0.73 \\
			\cmidrule{2-4}
			& \multirow{3}*{Deep R} & $0$ & 5.24\\
			& & $5\times 10^{-4}$ & 1.95 \\
			& & $10^{-3}$ & 1.04 \\
			\midrule
			\multirow{9}*{MNIST} & \multirow{5}*{Grad R} & $5\times 10^{-3}$ & 25.71\\
			& & $10^{-2}$ & 17.94\\
			& & $5\times 10^{-2}$ & 5.63\\
			& & $10^{-1}$ & 3.06\\
			& & $2\times 10^{-1}$ & 1.38\\
			\cmidrule{2-4}
			& \multirow{4}*{Deep R} & $0$ & 37.14\\
			& & $10^{-2}$ & 13.30\\
			& & $5\times 10^{-2}$ & 4.18\\
			& & $2\times 10^{-1}$ & 1.10\\

			\bottomrule
		\end{tabular}
		\caption{Choice of penalty term for Grad R (target sparsity $p=0.95$) and Deep R (maximum sparsity $p=1$ for CIFAR-10, $p=0.99$ for MNIST) with corresponding final connectivity.}
		\label{tab:hyper-2}
	\end{table}
	
	\section{Detailed Setting of Hyperparameters}\label{sec:hyper}
	
	The hyperparameters of Grad R and our implementation of Deep R are illustrated in Table \ref{tab:hyper-1} and Table \ref{tab:hyper-2} for reproducing results in this paper. Note that for Deep R, zero penalty still introduces low connectivity.

	\section{Range of Target Sparsity}\label{sec:range}
	Recall that the network weights $\bm{w}$ are determined by corresponding signs $\bm{s}$ and synaptic parameters $\bm{\theta}$, if the signs are randomly assigned and the distribution of $\theta_k$ is given, we are able to deduce the distribution of $w_k$. For this algorithm, we denote the cumulative distribution function (CDF) of $\theta_k$ is $F_{\theta_k}(\theta) := P[\theta_k<\theta]$. So we obtain 
	\begin{equation}\label{eq:after-relu}
		\begin{split}
			P[\mathrm{ReLU}(\theta_k)<\theta]&=P[\theta>\max\{\theta_k,0\}]\\
			&=
			\begin{cases}
				0, & \text{if } \theta \leq 0\\
				F_{\theta_k}(\theta), & \text{if } \theta > 0
			\end{cases}.
		\end{split}
	\end{equation}
	Since the distribution of sign is similar to the Bernoulli distribution, the CDF of $w_k$ can also be derived as 
	\begin{equation}\label{eq:weight}
		\begin{split}
			F_{w_k}(w) &:= P[w_k<w]\\
			&=P[s_k\mathrm{ReLU}(\theta_k)<w]\\
			&=\sum_{s\in\{1,-1\}}P[s_k\mathrm{ReLU}(\theta_k)<w\mid s_k=s]P[s_k=s]\\
			&=\frac12\left( P[\mathrm{ReLU}(\theta_k)<w]+P[-\mathrm{ReLU}(\theta_k)<w]\right)\\
			&=\frac12( P[\mathrm{ReLU}(\theta_k)<w]+1-P[\mathrm{ReLU}(\theta_k)<-w]\\
			&\quad -P[\mathrm{ReLU}(\theta_k)=-w])\\
			&=
			\begin{cases}
				\frac12(1-F_{\theta_k}(-w)), & \text{if } w < 0\\
				\frac12(1-P(\theta_k\leq 0)), & \text{if } w = 0\\
				\frac12(1+F_{\theta_k}(w)), & \text{if } w > 0\\
			\end{cases}.
		\end{split}
	\end{equation}
	Note that $F_{w_k}'(w)=\frac12F_{\theta_k}'(|w|)$, which indicates the probability density function (PDF) of $w_k$ is an even function. Therefore, the distribution of $w$ is symmetric with respect to the $w$-axis. We hope it to be unimodal, which matches our observation of weight distribution in SNNs (See Fig.~\ref{fig:param-no-prune}). So there should not be any peak in $F_{w_k}'(w)$ when $w>0$, otherwise we will get the other peak symmetric to it.
	
	Note that the requirement that $F_{w_k}'(w)=\frac12F_{\theta_k}'(|w|) (w>0)$ has no peak is equivalent to that $F_{\theta_k}'(w) (w>0)$ has no peak. As $F_{\theta_k}'(\cdot)$ is the prior of $\theta_k$, which is a Laplacian distribution where the PDF has only one peak at $\mu$, we must constrain $\mu\leq 0$. This explains why we require target sparsity $p\geq 0.5$.
	
	\bibliographystyle{named}
	\bibliography{ijcai21}

\end{document}